%% file: main.tex
\documentclass{article}
\usepackage{xcolor}
\usepackage[final]{corl_2026} % Uncomment for the camera-ready ``final'' version.
\usepackage{booktabs} 
\usepackage{hyperref}
\usepackage{tikz}
\usetikzlibrary{arrows.meta,calc,decorations.markings}

\usepackage{amsmath}   % equation, aligned, \eqref
\usepackage{amssymb}   % \mathbb, \varnothing
\usepackage{amsthm}    % lemma / proposition / corollary environments
\usepackage{enumitem}  % compact itemize in the appendix proofs

\newtheorem{proposition}{Proposition}

\newtheorem{corollary}{Corollary}
\theoremstyle{remark}
\newtheorem{remark}{\textbf{Remark}}

\title{Cheap Retrieval and Observation Credit Assignment from One Quantity: The Value of Information}

\author{
  Daniel Strauss$^{1}$ \qquad
  Nikolaos Sotirakis$^{1}$ \qquad
  Zhenshan Bing$^{2,1,3}$ \qquad
  Alois Knoll$^{1}$ \\[4pt]
  $^{1}$Technical University of Munich \qquad
  $^{2}$Nanjing University \qquad
  $^{3}$Harbin Institute of Technology \\[2pt]
  \texttt{daniel.strauss@tum.de}
}

\begin{document}
\maketitle

%===============================================================================

\begin{abstract}
    % Memory in reinforcement learning raises two problems. Attention finds the relevant item by scoring the whole store, which is costly for robots under latency and energy limits. And when a retrieved memory pays off, credit for acquiring it many steps earlier must cross the whole sequence. We show that both are governed by one quantity, the value of information, and propose to estimate it from archived rollouts with a learned world model. Read at the observation, it yields a provable lower bound on the optimal value of the acquiring action, which floors its TD target. Read at the decision, with retrieval cost expressed as negative reward, it defines the optimal retrieval set. 
    % The estimator is not yet implemented, but a memory system is tested, that can be trained with the information from the estimator. % <- must become better, new version of that sentence
    % %old version The estimator is work in progress; a pointer-based retrieval module, so far trained from task reward alone, illustrates how this set can be determined cheaply during inference.
    Memory in reinforcement learning raises two problems. Attention finds the relevant item by scoring the whole store, which is costly for robots under latency and energy limits. When a retrieved memory pays off, credit for acquiring it must cross the whole sequence. We argue that both are governed by one quantity, the value of information, and propose to estimate it from archived rollouts. We show that, read at the observation, it yields a lower bound on the optimal value of the acquiring action. This bound floors the action’s TD target and credits the observation that earned it. Read at the decision, it defines the optimal retrieval set. This work-in-progress submission does not yet test this algorithm. Still, a pilot retrieval system learns to approximate an optimization problem introduced by the proposed algorithm.   
\end{abstract}

% Two or three meaningful keywords should be added here
%reinforcement learning, partial observability, external memory, memory retrieval, retrieval-augmented policies, episodic memory, credit assignment, observation credit assignment, value of information, memory benchmarks and shortcuts

%===============================================================================

\input{SHORT_introduction}

% old versions
%\input{section_1_m}
%\input{figure_1}
%\input{figure_3}
%\input{section_1}
\input{v2_section_1}

\input{limitations}

\clearpage
% The acknowledgments are automatically included only in the final and preprint versions of the paper.
\acknowledgments{We used Claude and ChatGPT to create the latex code for figure \ref{fig:voi}, to assist formalize the proofs in the appendix, to assist with locating related literature, and to edit the text. All results were verified by the authors, who take full responsibility for the content.}

%===============================================================================

% no \bibliographystyle is required, since the corl style is automatically used.
\bibliography{bibliography}  % .bib

\newpage
\appendix

\input{v2_appendix}
\input{supplementary}

\end{document}

%% file: SHORT_introduction.tex
\section{Introduction}

% TODOs
% improfe language

% alternative: robot has to grap 3d object and can look at it before manipulating it. 
Consider a robot replacing a control board inside an unfamiliar appliance. Mechanically identical connectors attach behind the service panel frame, so before unplugging them the robot has to move its camera to look, and once the board is out that arrangement cannot be observed again. Hundreds of steps later each cable has to go back to the right socket, and each decision rests on one frame among thousands. Attending over all stored frames at every step to find it is what makes memory expensive on a robot.
Harder still is learning which observations to make in the first place: Transformer agents fail on long-horizon tasks when they must first act to observe a cue ~\citep{ni2023transformers}; we call this the observation credit assignment problem (O-CAP). Both problems require knowing which information matters at which decision before retrieving it: its information value, following~\citet{howard1966information}. We propose a way to estimate information value, which is the increase in value of decisions under this information, from archived rollouts. Then we propose a way to use this value to both credit actions that gather useful information in POMDPs and to retrieve useful memories.

Mechanisms that avoid this search compress instead and so commit to what is kept before knowing which decision it serves: recurrent states \citep{elman1990finding, hochreiter1997long, chung2014empirical}, still strong POMDP baselines \citep{ni2021recurrent}, recurrent world models~\citep{hafner2023mastering}, recurrent-transformer hybrids \citep{cherepanov2026recurrent, gupta2026memo}. Another way is through hierarchical or chunked attention \citep{lampinen2021towards, cho2024spatially}.
External memory models decouple storage capacity from the size of the policy's active representation, so that the history is kept with limited compression \citep{graves2014neural, graves2016hybrid, yang2016lie, parisotto2017neural}, including recent policies that retrieve VLA-generated tokens \citep{shi2026memoryvla}. 
% The below part of the paragraph can be cut to save space.
This moves the problem to retrieval: sparse retrieval gives unselected items no gradient, dense attention's cost grows with the store, and consolidation \citep{shi2026memoryvla} may discard what matters later.
General credit assignment methods redistribute delayed reward towards contributing events \citep{arjona2019rudder} or estimate the influence of earlier actions from later trajectory information \citep{harutyunyan2019hindsight, mesnard2020counterfactual, li2026counterfactual}. Temporal Value Transport (TVT), closest to our setting, does so via a later memory retrieval \citep{hung2019optimizing}. We instead define credit as the information value of an observation and obtain a provable floor on the TD target. TVT may implicitly learn the form of the information value derived here. 
\citet{raposo2021synthetic} assign credit to states whose presence predicts later reward and add this credit to the environment reward, which, conceptually, could double-count reward and thus overshoot the action value \citep{raposo2021synthetic}. Similar to \citep{raposo2021synthetic}, we credit actions based on future predictions of the reached observation, but by the expected improvement of later decisions rather than by predicted reward. As a result, as shown later, we do not overshoot the Q-value for exact quantities. Unlike \citep{raposo2021synthetic}, we only aim to improve credit assignment for information gathering and not for manipulating the environment.

%% file: v2_section_1.tex
% todo correct it, only add it if you have space
% The proposed method is based on two principles:
% \begin{enumerate}
%   \item By expectation an action can only be as valuable as available information allows to predict its return.
%   \item Observing and recalling are one problem, which occurs at two points in time.
% \end{enumerate}
\input{v2_figures/lower_bound_b}
\section{Consuming the Value of Information}
%$\mid B$ means that this prediction is conditioned on $H_B$ and that the observation of $H_B$ turns $U$ from a random variable into an expectation. 
Consider a POMDP in which $x$ is relevant information for a decision at $t$. It is possible that the exploring agent will encounter both $x$ and the decision in very few episodes. This is especially the case if $x$ can only be gathered by accumulating observations from time steps that are distal to one another and to $t$. %todo think about adding a figure
To estimate the contribution of $x$ to predicting the value of a state at $t$, we need to compare predictions made with and without $x$. We do not need to collect separate trajectories in which $x$ was not acquired. Instead, trajectories in which $x$ was observed can be reused under different knowledge states by deliberately masking $x$ from the recorded history.
To address sparse $x$ encounters we learn the action value at the decision by off-policy Q-learning on all archived histories, which reuses every encounter with $x$ rather than only those of the current policy. We subsequently estimate information value by regressing this action value onto histories masked to the respective knowledge states.

% add this if you have enough space
%This separates learning environmental relationships from estimating their value under an optimal policy, allowing us to attribute information value independently of the current actor's ability to exploit that information.

\paragraph{Estimating information value}
Let $\mu$ be the policy of the exploring agent during learning, which gives every
action a positive probability. We use $\pi$ to denote hypothetical policies. Let $H_t=(h_0,h_1,\dots,h_t)$ denote the recorded history, where $h_s=(a_{s-1}, r_{s-1}, o_s)$ contains the information acquired at timestep $s$. Write $[t]:=\{0,\dots,t\}$. For index set $A$ let $H_A$ denote the set of history items, that occur at a timestep in $A$ and $A \subseteq B$ means $H_A \subseteq H_B$. Index sets may also contain $\alpha_s$, referring to the action part $a_{s}$ of $h_{s+1}$.
% in a previous version actions from [t] supsert A where kept, but I think that this is not necessary
Following~\citet{howard1966information}, we define the value with information $H_A$ at decision $t$:
\begin{equation}\label{eq:def_U}
    \begin{aligned}
        U_t(A \mid B) := \mathbb{E}_{\pi_A}[G_t \mid H_B], \quad \pi_A \in \arg\max^\prec_{\pi \in \Pi_{A,t}} \mathbb{E}_\pi[G_t \mid H_A], \\ \Pi_{A,t} := \big\{\pi : \pi_k = \mu_k \text{ for } k < t,\ \pi_t(\cdot \mid H_t) = \pi_t(\cdot \mid H_A)\big\}.
    \end{aligned}
\end{equation} % this term makes you gather information that is not relevant for the subsequet action
where $G_t:=\sum_{k\ge t}\gamma^{k-t}r_k$ is the discounted future return with discount $\gamma\in(0,1)$.
$\prec$ fixes an order on the action space, and the maximisers deterministically play the $\prec$-first maximising action at $t$.
$U_t(A \mid B)$ is a function of $H_B$: the value of choosing the action at $t$ from $H_A$, evaluated with the information $H_B$. This term is easier to interpret when $B \subseteq A$: We already know $H_B$ (e.g. whether water can be used to extinguish fire), what can we obtain for decision $t$ (e.g. a 100\$ bill catches fire) if we extend our knowledge to $H_A$ (e.g. $H_B \cup \{ \text{where to find water}\}$).   

Given $A \subseteq [t]$, we write:
% todo: check if this definition of Q fits the rest
\begin{equation}\label{eq:def_Qstar}
    \begin{aligned}
Q_t^*(H_A,a) &:=
\mathbb E_{\mu, a_t \leftarrow a}\!\left[ r_t+\gamma U_{t+1}([t+1]\mid[t+1]) \,\middle|\, H_A\right],\\
(\mathcal TQ)_t(H_{t},a)&:=\mathbb{E}_{\mu, a_t \leftarrow a}\Big[r_{t}+\gamma\max_{a'}Q(H_{t+1},a')\ \Big|\ H_{t}\Big]
    \end{aligned}
\end{equation}
for the optimal action value and the Bellman optimality operator with fixed point $Q^*$ \citep{bellman1957markovian}. $\mu, a_t\leftarrow a$ denotes that until $t-1$ actions were sampled from $\mu$ and at $t$ action $a$ was chosen.

The marginal value of information $H_T$ at decision $t$ is then
\begin{equation}\label{eq:Lambda}
\Lambda_t(T; A \mid B) := U_t(A\cup T\mid B)-U_t(A\mid B)\;.
\end{equation}
Thus, information value is determined by the change in achievable return when $H_T$ is made available, rather than by the return of the current policy.

% We first estimate the optimal action value. We fit Q̂_t(H_t, a) by off-policy Q-learning on the archive D, regressing onto the Bellman target

% r_t + γ max_{a’} Q̂_{t+1}(H_{t+1}, a’), (τ, t) ∼ D. (4)

% Since Q_t(H_A, a) = E_µ[Q_t(H_t, a) | H_A] (Eq. (16)), we obtain the action value under a knowledge state H_A by regressing, for randomly sampled A ⊆ [t] and every action a,

% Q̂_t(H_A, a) onto Q̂_t(H_t, a). (5)

% U then follows from one more regression. By Prop. 2(a), U_t(A | B) is the expected value of the action chosen from H_A, evaluated with H_B. We therefore estimate Û(A | B) := u_{θ*U}(H_B, A), where ρ_t mixes credit pairs B = [s − 1] ∪ {α{s−1}} ⊆ A and retrieval pairs B = [t], and we regress using a loss ℓ′”

We estimate $U_t(A \mid B)$ from estimates of $Q^*_t(H_A, a)$. We first fit $\hat Q_t(H_t, a_t)$ by Q-learning on the archive $\mathcal D$, regressing it onto the Bellman target
\begin{equation}\label{eq:Qtarget}
r_t + \gamma \max_{a'} \hat Q_{t+1}(H_{t+1}, a'), \qquad (\tau, t) \sim \mathcal D.
\end{equation}
Since $Q^*_t(H_A, a) = \mathbb{E}_\mu[Q^*_t(H_t, a) \mid H_A]$ (see Eq. \eqref{eq:fixed-a}), we obtain the action value under a knowledge state $H_A \subseteq H_t$ by regressing, for randomly sampled $A \subseteq [t]$, $\hat Q_t(H_A, a)$ onto $\hat Q_t(H_t, a)$.
By Prop. \ref{prop:u_to_q}(a), $U_t(A \mid B)$ is the value of the action chosen from $H_A$, evaluated with $H_B$. We estimate it as $\hat U(A \mid B) := u_{\theta_U^*}(H_B, A)$, regressing on pairs $(A, B)$ drawn from a distribution $\rho_t$ that mixes credit pairs, $B = [s-1] \cup \{\alpha_{s-1}\} \subseteq A$, and retrieval pairs $B = [t]$, with a loss $\ell'$:
\begin{equation}\label{eq:Ufit}
\theta_U^* = \arg\min_{\theta_U} \mathbb{E}_{(\tau,t)\sim\mathcal{D},\ (A,B)\sim\rho_t}\Big[\ell'\Big(u_{\theta_U}(H_B, A),\ \hat Q_t\big(H_t, \arg\max\nolimits^{\prec}_a \hat Q_t(H_A, a)\big)\Big)\Big].
\end{equation}

\paragraph{Using the estimates and training}

For a given trajectory $\tau$ for decision at time $t$ we will assign this information credit $b$ to action $a$, which gathered the information from time $s$. Let $B_s(a) := [s-1] \cup \{\alpha_{s-1} = a\}$ and $C_{s,t,R}(a) := ([t] \setminus (R \cup \{s\})) \cup \{\alpha_{s-1} = a\}$

% old action independent definition 
%\begin{equation}\label{eq:b}
%    b_{s,t} = \sup_{R \subseteq \{s+1,\dots,t\}} \Lambda_t(\{s\};[t]\setminus (\{s\} \cup R) \mid [s-1])
%\end{equation}

\begin{equation}\label{eq:b}
    b_{s,t}(a) := \sup_{R \subseteq \{s+1,\dots,t\}} \Lambda_t\Big(\{s\};\ \underbrace{C_{s,t,R}(a)}_{\text{available at } t \text{ without } h_s \text{ and } H_R \text{, but with } a_{s-1}} \,\Big|\, \underbrace{B_s(a)}_{\text{available at $s-1$, with $a_{s-1}$}}\Big).
\end{equation}
 
 This term finds the set of observations $R$ that come after $s$ and before $t$, such that the win of observing $s$ is maximal. The idea is that $R$ should catch all observations that are mostly redundant given $s$ and appear only after $s$, but at the same time let $s$ be combined with synergistic information. Thereby, if for a decision at $t$, there are many redundant trajectory items $s^{(1)} < s^{(2)} < \dots < s^{(m)}$ the earliest receives the credit for the information they share. % this may not work out, proof? 
Additionally, if $\Lambda_t(\{s\};[s-1] \mid [s-1]) \approx 0$, but there is $s<u<t$ such, that $\Lambda_t(\{s\};[s-1] \cup \{u\} \mid [s-1]) \gg 0$, the synergistic contribution for $s$ and $u$ should be credited for both (Fig.~\ref{fig:voi}). 
% a problem is s_1 and s_2 are redindant and synergistic, the redundant credity may get lost. 
There may be many ways to approximate $b$ with a lower bound for faster computation, e.g. by restricting $R$ to steps spatially close to $s$.
% old LST action independent
% \begin{equation}
% \begin{aligned}
% L_{s,t}
% :=\mathbb E_{\mu}\left[\sum_{k=s-1}^{t-1} \gamma^{k-s+1}r_k \,\middle|\, H_{s-1},a_{s-1} \right]
%  + \gamma^{t-s+1} U_t(H_{s-1}\mid H_{s-1}).
% \end{aligned}
% \end{equation}
\begin{equation}\label{eq:def_Lst}
\begin{aligned}
L_{s,t}(a) := \mathbb{E}_\mu\Big[\sum_{k=s-1}^{t-1} \gamma^{k-s+1} r_k \,\Big|\, H_{s-1},\ a_{s-1} = a\Big] + \gamma^{t-s+1}\, U_t\Big([s-1] \cup \{\alpha_{s-1}\} \,\Big|\, [s-1] \cup \{\alpha_{s-1}\}\Big).
\end{aligned}
\end{equation}
$L_{s,t}(a)$ is a lower bound for the attainable return, if at $t$ the policy must ignore the information collected by action $a_{s-1}$.   
By Corollary \ref{corr:L_b_bound} (App. \ref{app:acq_new}) $\gamma^{t-s+1} b_{s,t}(a_{s-1}) + L_{s,t}(a_{s-1})$ is a lower bound of the optimal TD value at $H_{s-1},a_{s-1}$.  

\begin{equation}\label{eq:update}
    \mathcal (T_b Q)_{s-1}(H_{s-1}, a_{s-1}) =   \max\big\{(\mathcal TQ)_{s-1}(H_{s-1},a_{s-1}),~ \max_t ~\gamma^{t-s+1} b_{s,t}(a_{s-1}) + L_{s,t}(a_{s-1})  \big\} 
\end{equation}

This is similar in spirit to the lower-bound targets of~\citet{he2016learning} and \citet{abraham2026long}. 
The plain $n$-step floor (App.~\ref{app:acq_new}, Remark~\ref{rem:quick_note}) is at least as tight for exact quantities, but its estimate can carry the variance of everything $\mu$ acquires after $s$, which the maximum in Eq.~\eqref{eq:update} turns into upward bias. The floor in Eq.~\eqref{eq:update} may avoid this variance, which we will evaluate in future work.
For retrieval, let each recalled memory add cost $C$, following cost-aware selection of observations \citep{howard1966information, krause2005near, greiner2002learning}. The optimal retrieval set is 
\begin{equation}\label{eq:m_star}
M_t^* = \arg\max_{M\subseteq [t]} \Lambda_t(M;\emptyset\mid [t]) - C\lvert M \rvert \overset{\text{(a)}}\approx \arg\max_{M\subseteq [t]} \Lambda_t(M;\emptyset\mid M^*_{t-1}) - C\lvert M \rvert. 
\end{equation} 

Approximation (a) updates the retrieved set once per environment step, conditioning on the previously held set. Alternatively, the update can run asynchronously from the policy. This approximation has yet to be empirically validated.
During inference, it is important to compute $\hat M_t^*$ cheaply.  
In the next paragraph, we propose one possible method to do this while exploiting temporal dependencies between neighboring time steps.
 % todo add appendix with new shapely? 

% TODO 

% TODO

% %this was already explained in overview
% During training, collected trajectories are appended to $\mathcal D$, and the world model and achievable-value estimator are updated to compute information \newer{values}. The resulting signals \new{$F_t$ and $\hat b_s$} can then be used to train a memory-retrieval mechanism and an external-action policy, respectively. We expect this policy to converge faster than \new{with return-based credit alone, as the floor moves credit to the acquiring transition in one update rather than over $t-s$ backups. For an on-policy critic, the floor acts as deliberate optimism, since it bounds $Q^*$ rather than} the value of the current policy\new{.} Their architectures and optimization algorithms can be chosen independently of the proposed information-value estimation method.

\paragraph{Memory Navigation}\label{p:mem_nav}
% old and shitty
% As a first investigation of a memory system, that should learn to approximate $M^*_t$, we implemented Memory Navigation, in which a persistent pointer moves relatively through a spatial indexed external memory. 
% An internal action head controls this ``Memory Ghost'', while only the retrieved memories enter a queue, which is then feed into a Transformer-XL Network. 
% During inference the retrieval head only costs one neuronal layer. 
% In this pilot implementation, the internal action head is trained by vanilla RL, and importantly the signal from $\hat \Lambda$ is not yet implemented.
% We evaluated Memory Navigation with PPO on MiniGrid-Memory-S17Random. After pretraining on the smaller S9 environment with a context of 64 tokens, models were fine-tuned with a context of only 5 tokens. Memory Navigation reached a mean return of 0.95 after 0.59M environment steps, compared with 0.94M for Gaussian absolute-position retrieval and 1.38M for Transformer-XL without a retrieval head. The full-context pretrained baseline remained faster. These experiments evaluate the retrieval architecture only; neither the information-value estimator nor the floored TD target is used yet.
In this paragraph, we propose a mechanism to approximate $M_t^*$ in a computationally feasible manner and test whether it can find the relevant memory with a small context. It shares similarities with prior memory architectures \citep{yang2016lie, parisotto2017neural, zaremba2015reinforcement}, namely relative addressing, memory indexed by location, and discrete memory access trained with RL. We do not yet implement $\hat\Lambda$, but use plain PPO \citep{schulman2017proximal} as optimization signal. Orthogonal to approximation (a), we expect $M^*_t$ to be mostly a small, spatiotemporal local edit of $M^*_{t-1}$ in many robotic environments. Our proposed algorithm Memory Navigation exploits this with a persistent pointer, the Memory Ghost, that moves by relative steps through an external memory indexed by where and when each memory was recorded. Recalling thereby resembles observing: a second, internal action head moves the Ghost through memory, which is structured like the actual environment, and the memories it reaches enter the Transformer-XL \citep{dai2019transformer} context like observations. $\hat\Lambda_t(\hat M_t;\emptyset\mid [t]) - C\lvert \hat M_t \rvert$ would be its reward, but neither it nor the floor is used yet. After pretraining on MiniGrid-Memory-S9 \citep{chevalier2023minigrid} with a 64-token context and fine-tuning on S17Random with 5 tokens, Memory Navigation reached a mean return of 0.95 after 0.59M steps, versus 0.94M for Gaussian absolute-position retrieval (6 seeds each) and 1.38M without a retrieval head (3 seeds). Only the full-context pretrained model was faster. The baseline without a retrieval head still solved the task despite its 3 or 5 past token context, by cue-dependent waiting and storing the cue in its temporal embedding or by cue-dependent turning to the right and back every 2 steps and then decoding the cue at decision from recent observation. Code and rollout videos are provided in App.~\ref{app:supplementary}.

%% file: v2_figures/lower_bound_b.tex
% ---- requires: \usepackage{tikz}, \usetikzlibrary{arrows.meta,calc,decorations.markings} ----
\definecolor{credit}{HTML}{D55E00}
\definecolor{recall}{HTML}{0072B2}
\definecolor{bellman}{HTML}{009E73}
\begin{figure}[t]
\centering
\begin{tikzpicture}[
  x=1cm,y=1cm,font=\scriptsize,
  st/.style={circle,fill=black!80,inner sep=0pt,minimum size=5pt},
  nt/.style={circle,draw=black!40,fill=white,inner sep=0pt,minimum size=5pt},
  term/.style={rectangle,draw=black!80,inner sep=1.8pt},
  nterm/.style={rectangle,draw=black!40,text=black!55,inner sep=1.8pt},
  taken/.style={-{Stealth[length=4pt]},black!85,thick},
  untaken/.style={-{Stealth[length=3.5pt]},black!35},
  bell/.style={-{Stealth[length=4pt]},bellman,thick},
  lt/.style={black!85},          % label of a taken action
  lu/.style={black!50},          % label of an action not taken
  obs/.style={rectangle,draw=black!55,fill=white,inner sep=1.5pt,rounded corners=1pt},
  obsu/.style={obs,draw=black!35,text=black!50},
  dual/.style={black!55,semithick},
  target/.style={densely dotted,thick,-{Stealth[length=3.5pt]}},
]
\def\bh{1.0}  % bar height per unit of value
\def\bw{0.17} % bar width
% bar: #1 x-center, #2 baseline y, #3 Q*, #4 top of the grey part (= #3 if no information-value part),
%      #5 label, #6 label style, #7 colour of the information-value part (credit or recall)
\newcommand{\vbar}[7]{%
  \fill[black!18] ($(#1,#2)+(-\bw/2,0)$) rectangle ++(\bw,#4*\bh);
  \fill[#7!80] ($(#1,#2)+(-\bw/2,#4*\bh)$) rectangle ++(\bw,{(#3-#4)*\bh});
  \draw[black!80,line width=0.35pt] ($(#1,#2)+(-\bw/2,0)$) rectangle ++(\bw,#3*\bh);
  \node[#6,anchor=north,inner sep=1.5pt] at (#1,#2) {#5};}
% many steps: solid -- dotted -- solid arrow; #1 style, #2 dot width, #3 from, #4 to
\newcommand{\many}[4]{%
  \draw[#1,-] (#3) -- ($(#3)!0.28!(#4)$);
  \draw[#1,-,dotted,line width=#2] ($(#3)!0.37!(#4)$) -- ($(#3)!0.63!(#4)$);
  \draw[#1] ($(#3)!0.72!(#4)$) -- (#4);}

% ---------------- MDP and recorded trajectory ----------------
\node[st,label={[black!85]left:spawn}] (s0) at (0,0) {};
\node[st,label={[black!85]below:$s{-}1$}] (sm) at (1.2,0) {};
\node[st] (s) at (1.9,0.8) {};
\node[st] (sp1) at (2.6,0) {};
\node[st,label={[black!85]below:$k{-}1$}] (km) at (3.9,0) {};
\node[nt] (sp) at (4.6,0.8) {};
\node[st] (k) at (5.3,0) {};
\node[st,label={[black!85]below:$u{-}1$}] (um) at (6.6,0) {};
\node[st] (u) at (7.3,0.8) {};
\node[st] (up1) at (8.0,0) {};
\node[st,label={[black!85]below left=-1pt:$t$}] (t) at (9.3,0) {};
\node[st] (l1) at (10.0,0.5) {};
\node[term] (r1) at (12.1,0.5) {$r{=}1$};
\node[nt] (l2) at (10.0,-0.5) {};
\node[nterm] (r0) at (12.1,-0.5) {$r{=}0$};

\many{taken}{1.3pt}{s0}{sm}
\draw[taken] (sm) -- node[lt,left=1pt,pos=0.45]{look} (s);  \draw[taken] (s) -- (sp1);
\draw[untaken] (sm) -- node[lu,below,pos=0.66]{skip} (sp1);
\many{taken}{1.3pt}{sp1}{km}
\draw[untaken] (km) -- node[lu,left=1pt,pos=0.45]{look} (sp);  \draw[untaken] (sp) -- (k);
\draw[taken] (km) -- node[lt,below,pos=0.66]{skip} (k);
\many{taken}{1.3pt}{k}{um}
\draw[taken] (um) -- node[lt,left=1pt,pos=0.45]{look} (u);  \draw[taken] (u) -- (up1);
\draw[untaken] (um) -- node[lu,below,pos=0.66]{skip} (up1);
\many{taken}{1.3pt}{up1}{t}
\draw[taken] (t) -- node[lt,above left=-2pt,pos=0.72]{left} (l1);
\draw[untaken] (t) -- node[lu,above right=-2pt,pos=0.72]{right} (l2);
\many{taken}{1.3pt}{l1}{r1.west}
\many{untaken}{1.0pt}{l2}{r0.west}

% observations glued to the nodes where they are made
\node[obs,anchor=south west] (o1) at ($(s)+(0.2,0.12)$) {obs: $x_1$};
\draw[black!55] (o1.south west) -- (s);
\node[obsu,anchor=south west] (o2) at ($(sp)+(0.2,0.12)$) {obs: $x_1$};
\draw[black!35] (o2.south west) -- (sp);
\node[obs,anchor=south west] (o3) at ($(u)+(0.2,0.12)$) {obs: $x_2$};
\draw[black!55] (o3.south west) -- (u);

% ---------------- Q* bars, two per decision ----------------
% Illustrative POMDP: each dotted segment = 10 steps, look detour = 2 steps (cost c=0.05; c'=0.4 at k-1), skip = 1 step,
% reward 1 ten steps after `left'; at t the rewarded side depends on x1 and x2 jointly (guess = 1/2 without both).
% => t-s+1 = 35, t-u+1 = 12, b_{s,t} = b_{u,t} = 0.452, L_{s,t} = 0.228, L_{u,t} = 0.351; the floors L+gamma^{t-s+1}b are tight.
%    U_t({}|{}) = 0.452, Lambda_t({s,u};{}|{}) = U_t({s,u}|{}) - U_t({}|{}) = 0.452.
% look bars: grey part = L, orange part = gamma^{t-s+1} b (information value of the observation; sum = Q*).
% left bar at t: grey part = U_t({}|{}), blue part = Lambda_t({s,u};{}|{}) (information value of recall; sum = Q*).
% (gamma = 0.99, c = 0.05, c' = 0.4 and the step counts only fix the bar heights; they are not part of the paper.)
\def\by{-1.7}
\vbar{1.5}{\by}{0.547}{0.228}{look}{lt}{credit}   \vbar{2.3}{\by}{0.325}{0.325}{skip}{lu}{credit}
\vbar{4.2}{\by}{0.266}{0.266}{look}{lu}{credit}   \vbar{5.0}{\by}{0.673}{0.673}{skip}{lt}{credit}
\vbar{6.9}{\by}{0.752}{0.351}{look}{lt}{credit}   \vbar{7.7}{\by}{0.405}{0.405}{skip}{lu}{credit}
\vbar{9.3}{\by}{0.904}{0.452}{left}{lt}{recall}   \vbar{10.1}{\by}{0.0}{0.0}{right}{lu}{recall}
\foreach \x/\n in {1.9/$s{-}1$,4.6/$k{-}1$,7.3/$u{-}1$,9.7/$t$} {
  \draw[black!30] (\x-0.62,\by) -- (\x+0.62,\by);
  \node[black!60,anchor=north] at (\x,\by-0.25) {\n};}

% TD targets: information value of the observation (orange, floors the target of look) and Bellman backup (green)
\draw[credit,target] (s.south) -- (1.5,\by+0.547*\bh+0.03);
\draw[credit,target] (u.south) -- (6.9,\by+0.752*\bh+0.03);
\node[text=credit,anchor=east] at (1.53,-0.65) {$\gamma^{t-s+1}b_{s,t}(\text{look})$};
\node[text=credit,anchor=east] at (6.88,-0.65) {$\gamma^{t-u+1}b_{u,t}(\text{look})$};
\draw[bellman,target] (t.south) -- (9.3,\by+0.904*\bh+0.03);

% Bellman backups: from the goal back along the recorded branch, step by step
\draw[bellman,thick] (r1.north) to[out=90,in=0] (11.8,0.9) -- (11.4,0.9);
\draw[bellman,dotted,line width=1.3pt] (11.25,0.9) -- (10.85,0.9);
\draw[bell] (10.7,0.9) -- (10.4,0.9) to[out=180,in=60] (l1.north east);
\draw[bell] (l1.south) to[bend left=35] (t.east);
\node[text=bellman,anchor=south] at (10.95,0.95) {$r_{j}+\gamma\max_{a'}Q^*(H_{j+1},a')$};

% ---------------- forked dual arrow (no junction) ----------------
\coordinate (P) at (9.3,0.75);
\tikzset{marks/.style={postaction={decorate,decoration={markings,
   mark=at position #1 with {\arrow[credit,line width=0.9pt]{Stealth[length=5pt]}},
   mark=at position #1+0.07 with {\arrowreversed[recall,line width=0.9pt]{Stealth[length=5pt]}}}}}}
\draw[dual] (t.north) -- (P);
\draw[dual,marks=0.46] (P) .. controls (9.3,3.0) and (1.9,3.1) .. (s.north);
\draw[dual,marks=0.35] (P) .. controls (9.3,2.0) and (7.3,2.1) .. (u.north);
\draw[credit,-{Stealth[length=4.5pt]}] ($(s.north)+(0,0.25)$) -- (s.north);
\draw[credit,-{Stealth[length=4.5pt]}] ($(u.north)+(0,0.25)$) -- (u.north);
\draw[recall,-{Stealth[length=4.5pt]}] ($(t.north)+(0,0.3)$) -- (t.north);
%\node[text=credit,anchor=south] at (5.,2.55) {$b_{s,t}(\text{look}) = \Lambda_t(\{s\}; C_{s,t,\emptyset} \mid B_s)$};
\node[text=credit,anchor=south] at (1.,2.) {$b_{s,t}(\text{look}) = \Lambda_t(\{s\}; C_{s,t,\emptyset} \mid B_s)$};
\node[text=credit,anchor=east] at (7.58,1.5) {$b_{u,t}(\text{look}) = \Lambda_t(\{u\}; C_{u,t,\emptyset} \mid B_u)$};
\node[text=recall,anchor=west] at (9.,1.95) {$\Lambda_t(\{s,u\};\emptyset\mid[t])$};

% ---------------- legend ----------------
\begin{scope}[font=\scriptsize]
  \node[anchor=west] at (-0.6,-2.6) {\tikz{\draw[taken] (0,0)--(0.4,0);}~recorded\quad
  \tikz{\draw[black!85,thick] (0,0)--(0.15,0); \draw[black!85,dotted,line width=1.3pt] (0.22,0)--(0.42,0); \draw[taken] (0.49,0)--(0.65,0);}~recorded, many steps\quad
  \tikz{\draw[untaken] (0,0)--(0.45,0);}~not taken\quad
  \tikz{\draw[black!55] (0,0)--(0.55,0); \path[postaction={decorate,decoration={markings,mark=at position 0.35 with {\arrow[credit,line width=0.9pt]{Stealth[length=5pt]}},mark=at position 0.8 with {\arrowreversed[recall,line width=0.9pt]{Stealth[length=5pt]}}}}] (0.55,0)--(0,0);}~dual arrow: \textcolor{credit}{credit}\,/\,\textcolor{recall}{recall}\quad
  \tikz{\draw[bellman,thick] (0.65,0)--(0.5,0); \draw[bellman,dotted,line width=1.3pt] (0.43,0)--(0.23,0); \draw[bell] (0.16,0)--(0,0);}~\textcolor{bellman}{Bellman backups}};
  \node[anchor=west] at (-0.6,-3.0) {\tikz{\draw[black!80,line width=0.35pt,fill=black!18] (0,0) rectangle (0.13,0.28);} $Q_j^*(H_{j},a)$ of each action $a$\quad
  \tikz{\fill[credit!80] (0,0) rectangle (0.13,0.28);}~information value of observation\quad
  \tikz{\fill[recall!80] (0,0) rectangle (0.13,0.28);}~information value of recall\quad
  \tikz{\draw[credit,target] (0,0)--(0.4,0);}\,\tikz{\draw[bellman,target] (0,0)--(0.4,0);}~TD target};
\end{scope}
\end{tikzpicture}
\caption{Information value on a recorded POMDP trajectory. Decision $t$ needs $x_1$ (at $s$) and $x_2$ (at $u$) jointly; the look at $k{-}1$ would repeat $x_1$ but is taken by $\mu$ with probability 0, so $R = \emptyset$ attains the sup in Eq. \eqref{eq:b}. 
Bellman backups carry the reward back step by step. From $t$, where the information changes the action, the dual arrow jumps back. Along orange it credits the acquiring actions, flooring their TD targets; $x_1$ is credited since $x_2$ is in its base set. Along blue it values recalling $s$ and $u$ jointly.}
\label{fig:voi}
\end{figure}
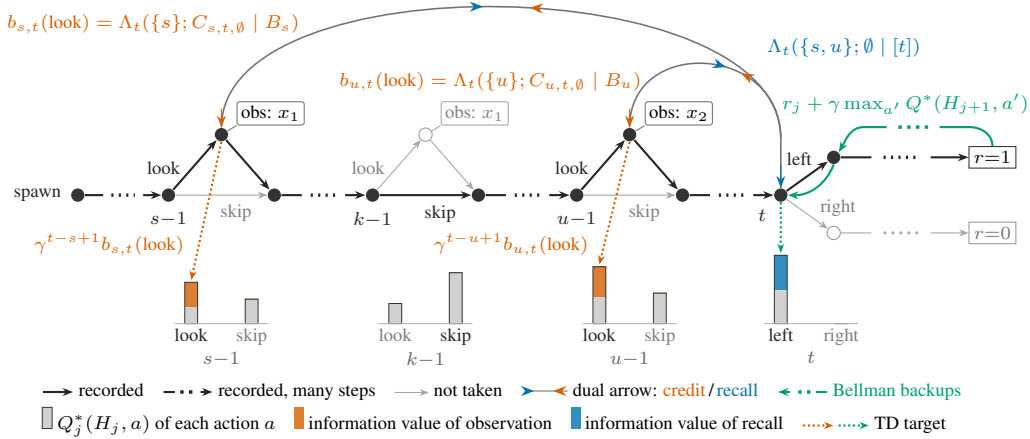

%% file: limitations.tex
\section{Limitations}

\paragraph{Validation.} Neither the information-value estimator nor the floored target is implemented yet, so whether the split floor has lower variance and improves learning remains a hypothesis (Remark~\ref{rem:quick_note}). We prove a floor on the TD target but do not relate information value to causal notions of contribution, such as whether an observation is necessary for the chosen action.
\paragraph{Estimation.} The bounds hold for exact quantities. Regressed estimators carry errors, maximizing over actions can exploit them, and the supremum over $R$ in Eq. \eqref{eq:b} is biased upward when taken over noisy estimates; each can push the estimated floor above $Q^*$. Restricting $R$ to steps near $s$, e.g. in space, reduces both cost and bias. Information value is defined under the exploration policy $\mu$. Thus, actions taken after $s$ may reveal the masked $h_s$ and lower its credit. Training still needs full histories; the savings arise only at inference. 
\paragraph{Modeling choices.}
Indexing information by timestep is partly arbitrary. Neighboring observations may share redundant, synergistic and unique information. Since the supremum in Eq.~\eqref{eq:b} selects a single $R$, $b_{s,t}$ credits the larger of these components rather than their sum; if $h_s$ and a later $h_u$ are both redundant and synergistic and the synergy dominates, the redundant part is credited to neither. $\Lambda_t$ could be used to separate these components.
Open questions remain: retrieval chains in which one memory is needed to locate another; and whether $M^*_t$ and $b_{s,t}$ arise from one common object. In future work, an assessment of whether a collected trajectory adds useful training data could yield an exploration information value to train a dedicated exploration policy.

%% file: v2_appendix.tex
\appendix

\section{Expressions of U}\label{s:exp_u}

As in the main text, see Eq.~\eqref{eq:def_U}, the definition of $U_t$ is

\begin{equation}
    \begin{aligned}
        U_t(A \mid B) := \mathbb{E}_{\pi_A}[G_t \mid H_B], \quad \pi_A \in \arg\max^\prec_{\pi \in \Pi_{A,t}} \mathbb{E}_\pi[G_t \mid H_A], \\ \Pi_{A,t} := \big\{\pi : \pi_k = \mu_k \text{ for } k < t,\ \pi_t(\cdot \mid H_t) = \pi_t(\cdot \mid H_A)\big\}.
    \end{aligned}
\end{equation}

Let us first consider the inner term. 

\begin{proposition}
    
Given $A \subseteq [t]$,

\begin{equation}\label{eq:U_Q}
\begin{aligned}
\sup_{\pi \in\Pi_{A,t}}\mathbb E_\pi\left[G_t \mid H_A \right]
&\overset{\text{(a)}}=\max_{a}\mathbb{E}_{\mu,a_t \leftarrow a}\Big[r_t+\gamma\sup_\pi\mathbb{E}_\pi[G_{t+1}\mid H_{t+1}]\;\Big|\;H_A\Big]\\
&\overset{\text{(b)}}=\max_{a}\mathbb{E}_{\mu,a_t \leftarrow a}\big[r_t+\gamma U_{t+1}([t{+}1]\mid[t{+}1])\;\big|\;H_A\big]\\
&\overset{\text{(c)}}=\max_{a}Q^*_t(H_A,a).
\end{aligned}
\end{equation}

\begin{proof}
\textbf{(a)}

Write $G_t = r_t + \gamma G_{t+1}$. For any $\pi \in \Pi_{A,t}$, the history up to $t$ follows $\mu$ and $a_t$ is drawn from $\pi_t(\cdot \mid H_A)$ alone. The tower property (using $H_A \subseteq H_{t+1}$) therefore gives
\begin{equation*}
    \begin{aligned}
\mathbb{E}_\pi[G_t \mid H_A]
= \sum_a \pi_t(a \mid H_A)\, \mathbb{E}_{\mu, a_t \leftarrow a}\Big[r_t + \gamma\, \mathbb{E}_\pi[G_{t+1} \mid H_{t+1}] \,\Big|\, H_A\Big].
    \end{aligned}
\end{equation*}

The constraint in $\Pi_{A,t}$ fixes $\pi$ before $t$ to $\mu$ and restricts $\pi_t$. The later rules $\pi_{t+1}, \pi_{t+2}, \dots$ are free and may depend on the full $H_{t+1}$. So the supremum splits into a choice of $\pi_t$ and a choice of the continuation: 
% TODO: how did the sup go into the expectation 

\begin{equation*}
    \begin{aligned}
\sup_{\pi \in \Pi_{A,t}} \mathbb{E}_\pi[G_t \mid H_A]
=& \sup_{\pi_t,\pi} \sum_a \pi_t(a \mid H_A)\, \mathbb{E}_{\mu, a_t \leftarrow a}\Big[r_t + \gamma \mathbb{E}_{\pi}[G_{t+1} \mid H_{t+1}] \,\Big|\, H_A\Big] \\
\overset{\ast}=& \sup_{\pi_t} \sum_a \pi_t(a \mid H_A)\, \mathbb{E}_{\mu, a_t \leftarrow a}\Big[r_t + \gamma \sup_{\pi} \mathbb{E}_{\pi}[G_{t+1} \mid H_{t+1}] \,\Big|\, H_A\Big] \\
=& \max_a \mathbb{E}_{\mu, a_t \leftarrow a}\Big[r_t + \gamma \sup_{\pi} \mathbb{E}_{\pi}[G_{t+1} \mid H_{t+1}] \,\Big|\, H_A\Big].
    \end{aligned}
\end{equation*}

$\pi$ may depend on $H_{t+1}$, which contains $H_A$ and $a_t$, and
$\mathbb{E}_\pi[G_{t+1}\mid H_{t+1}]$ depends only on $H_{t+1}$ and on $\pi$ after $H_{t+1}$. Thus, $\ast$ follows from the interchangeability principle \citep{shapiro2017interchangeability}. 

\textbf{(b)}
\begin{equation}
    \begin{aligned}
U_{t+1}([t+1] \mid [t+1])
&= \mathbb{E}_{\pi_{[t+1]}}[G_{t+1} \mid H_{t+1}]
&& \text{(Eq.~\eqref{eq:def_U})} \\
&= \sup_{\pi \in \Pi_{[t+1],t+1}} \mathbb{E}_\pi[G_{t+1} \mid H_{t+1}]
&& (\pi_{[t+1]} \text{ is a maximiser}) \\
&= \sup_\pi \mathbb{E}_\pi[G_{t+1} \mid H_{t+1}].
    \end{aligned}
\end{equation}
The last equality holds because, for $A = [t+1]$, the constraint on
$\pi_{t+1}$ in $\Pi_{[t+1],t+1}$ is vacuous, and the constraint
$\pi_k = \mu_k$ for $k \le t$ does not affect the conditional expectation
given $H_{t+1}$.

\textbf{(c)}
The last equality follows directly from the definition of $Q_t^*$ in Eq.~\eqref{eq:def_Qstar}.

\end{proof}

\end{proposition}

\begin{proposition}\label{prop:u_to_q}
    Let $A, B \subseteq [t]$ and $a^*_t(H_A) = \arg\max^\prec_a Q^*_t(H_A, a)$. Then
    \begin{equation}
        U_t(A \mid B) \overset{\text{(a)}}= \mathbb{E}_\mu\big[Q^*_t\big(H_{A \cup B}, a^*_t(H_A)\big) \,\big|\, H_B\big].
    \end{equation}
In particular,
\begin{equation}\begin{aligned}
 B \subseteq A:&\quad U_t(A \mid B) \overset{\text{(b)}}= \mathbb{E}_\mu\big[\max_a Q^*_t(H_A, a) \,\big|\, H_B\big], &&\text{"Learning"} \\
A \subseteq B:&\quad U_t(A \mid B) \overset{\text{(c)}}= Q^*_t\big(H_B, a^*_t(H_A)\big). &&\text{"Forgetting" }
\end{aligned}\end{equation}

\begin{proof}
\textbf{(a)}

Consider $\pi_A \in \arg\max^\prec_{\pi \in \Pi_{A,t}} \mathbb{E}_\pi\left [G_t \mid H_A \right ]$. After $t$, $\pi_A$ rules free and optimal, this gives  
\begin{equation*}
    \mathbb{E}_{\pi_A}[G_{t+1} \mid H_{t+1}] = U_{t+1}([t+1] \mid [t+1]).
\end{equation*}
Write $G_t = r_t + \gamma G_{t+1}$. Since $a^*_t(H_A)$ is a function of $H_t$, Eq.~\eqref{eq:def_Qstar} with $A = [t]$ gives
\begin{equation*}
\mathbb{E}_{\pi_A}[G_t \mid H_t]
= \mathbb{E}_{\mu,\, a_t \leftarrow a^*_t(H_A)}\big[r_t + \gamma U_{t+1}([t+1] \mid [t+1]) \,\big|\, H_t\big]
= Q^*_t\big(H_t, a^*_t(H_A)\big).
\end{equation*}
Up to $t$, $\pi_A$ follows $\mu$. Since $B \subseteq [t]$, the tower property gives
\begin{equation*}
    U_t(A \mid B)
= \mathbb{E}_{\pi_A}[G_t \mid H_B]
= \mathbb{E}_\mu\big[\mathbb{E}_{\pi_A}[G_t \mid H_t] \,\big|\, H_B\big]
= \mathbb{E}_\mu\big[Q^*_t\big(H_t, a^*_t(H_A)\big) \,\big|\, H_B\big].
\end{equation*}

Fix an action $a$. Since $H_{A \cup B} \subseteq H_t$, Eq.~\eqref{eq:def_Qstar} and the tower property give

\begin{equation}\label{eq:fixed-a}
\begin{aligned}
\mathbb{E}_\mu\big[Q^*_t(H_t, a) \,\big|\, H_{A \cup B}\big] &= \mathbb{E}_{\mu, a_t \leftarrow a}\big[Q^*_t(H_t, a) \,\big|\, H_{A \cup B}\big]\\
&= \mathbb{E}_{\mu, a_t \leftarrow a}\Big[\mathbb{E}_{\mu,\, a_t \leftarrow a}\big[r_t + \gamma U_{t+1}([t+1] \mid [t+1]) \,\big|\, H_t\big] \,\Big|\, H_{A \cup B}\Big] \\
&=\mathbb{E}_{\mu,\, a_t \leftarrow a}\Big[r_t + \gamma U_{t+1}([t+1] \mid [t+1])  \,\Big|\, H_{A \cup B}\Big] \\
&= Q^*_t(H_{A \cup B}, a).
\end{aligned}
\end{equation}

Now use $B \subseteq A \cup B$:

\begin{equation*}
\begin{aligned}
&\mathbb{E}_\mu\big[Q^*_t\big(H_t, a^*_t(H_A)\big) \,\big|\, H_B\big]\\
&= \mathbb{E}_\mu\Big[\mathbb{E}_\mu\big[Q^*_t\big(H_t, a^*_t(H_A)\big) \,\big|\, H_{A \cup B}\big] \,\Big|\, H_B\Big] && \text{(by law of total E., $B \subseteq A \cup B$)}\\
&= \mathbb{E}_\mu\Big[\sum_a \mathbf{1}\{a^*_t(H_A) = a\}\, \mathbb{E}_\mu\big[Q^*_t(H_t, a) \,\big|\, H_{A \cup B}\big] \,\Big|\, H_B\Big]  && \text{($a^*_t(H_A)$ is a function of $H_{A \cup B}$)} \\
&= \mathbb{E}_\mu\big[Q^*_t\big(H_{A \cup B}, a^*_t(H_A)\big) \,\big|\, H_B\big]. && \text{(\eqref{eq:fixed-a}, exactly one indicator is 1)}
\end{aligned}
\end{equation*}

\textbf{(b)}

Given $B \subseteq A$, $A\cup B = A$. It follows
\begin{equation}
    \mathbb{E}_\mu\big[Q^*_t\big(H_{A \cup B}, a^*_t(H_A)\big) \,\big|\, H_B\big] = \mathbb{E}_\mu\big[\max_a Q^*_t\big(H_{A}, a\big) \,\big|\, H_B\big]
\end{equation}

\textbf{(c)}
Given $A \subseteq B$, $A\cup B = B$. It follows
\begin{equation}
    \begin{aligned}
    &\mathbb{E}_\mu\big[Q^*_t\big(H_{A \cup B}, a^*_t(H_A)\big) \,\big|\, H_B\big]  
    =\mathbb{E}_\mu\big[Q^*_t\big(H_{B}, a^*_t(H_A)\big) \,\big|\, H_B\big] \\
    &=Q^*_t\big(H_{B}, a^*_t(H_A)\big) && \text{($Q^*_t\big(\dots\big)$ is a function of $H_B$)} 
\end{aligned}
\end{equation}

\end{proof}

\end{proposition}

\section{New Information Carries non-negative value}
This proposition is a simple case of the Blackwell theorem \citep{blackwell1953equivalent}.
\begin{proposition}\label{prop:black}
    Let $S, A, B \subseteq [t]$ with $B \subseteq A$. Then
    \begin{equation}
        U_t(S \cup A|B) \geq U_t(A|B)
    \end{equation}
        
\begin{proof} % TODO this proof is a simple case of blackwells theorem {[cite: Blackwell, 1953; Smallwood and Sondik, 1973]}\new{),

Since $B \subseteq A \subseteq S \cup A$, Proposition~\ref{prop:u_to_q}(b) and the tower property give

\begin{equation*}
    U_t(S \cup A \mid B)
= \mathbb{E}_\mu\big[\max_a Q^*_t(H_{S \cup A}, a) \,\big|\, H_B\big]
= \mathbb{E}_\mu\Big[\mathbb{E}_\mu\big[\max_a Q^*_t(H_{S \cup A}, a) \,\big|\, H_A\big] \,\Big|\, H_B\Big].
\end{equation*}

For the inner term, fix any action $a^{(0)}$. Then

\begin{equation*}    
\begin{aligned}
\mathbb{E}_\mu\big[\max_a Q^*_t(H_{S \cup A}, a) \,\big|\, H_A\big]
&\ge \mathbb{E}_\mu\big[Q^*_t(H_{S \cup A}, a^{(0)}) \,\big|\, H_A\big] \\
&= \mathbb{E}_\mu\Big[\mathbb{E}_\mu\big[Q^*_t(H_t, a^{(0)}) \,\big|\, H_{S \cup A}\big] \,\Big|\, H_A\Big]
&& \text{(Eq.~\eqref{eq:fixed-a} with } S \cup A\text{)} \\
&= \mathbb{E}_\mu\big[Q^*_t(H_t, a^{(0)}) \,\big|\, H_A\big]
&& \text{(tower, } A \subseteq S \cup A\text{)} \\
&= Q^*_t(H_A, a^{(0)})
&& \text{(Eq.~\eqref{eq:fixed-a} with } A\text{)}.
\end{aligned}
\end{equation*}

Since $a^{(0)}$ is arbitrary, the inner term is at least $\max_a Q^*_t(H_A, a)$. Proposition \ref{prop:u_to_q}(b) applies because $B \subseteq A$. Together this gives

$$U_t(S \cup A \mid B) \ge \mathbb{E}_\mu\big[\max_a Q^*_t(H_A, a) \,\big|\, H_B\big] = U_t(A \mid B). \qquad \square$$

\end{proof}
\end{proposition}

\section{Acquisition bound}\label{app:acq_new}

\begin{proposition}[Acquisition bound]\label{prop:acq}
Fix $t\ge s$, an action $a$ and $R \subseteq \{s+1, \dots, t\}$, let 
\begin{equation*}
\begin{aligned}
&B_s := B_s(a), \qquad\qquad C_R := C_{s,t,R}(a),\\
&L_{s,t,R}(a) := \mathbb{E}_\mu\Big[\sum_{k=s-1}^{t-1} \gamma^{k-s+1} r_k \,\Big|\, H_{s-1},\ a_{s-1} = a\Big] + \gamma^{t-s+1}\, U_t(C_R \mid B_s).
\end{aligned}
\end{equation*}
Then
\begin{equation}\label{eq:acq}
Q_{s-1}^*(H_{s-1},a)\;\ge\;L_{s,t,R}(a)+\gamma^{t-s+1}\,\Lambda_t\big(\{s\};C_R|B_s).
\end{equation}

\end{proposition}
\begin{proof}
\begin{equation}\label{eq:plain_floor}
    \begin{aligned}
        &Q_{s-1}^*(H_{s-1},a)\\
        &\ge \mathbb{E}_{\mu,\, a_{s-1} \leftarrow a}\Big[\sum_{k=s-1}^{t-1} \gamma^{k-s+1} r_k + \gamma^{t-s+1}\, U_t([t] \mid [t]) \,\Big|\, H_{s-1}\Big] && \text{(follow $\mu$ until $t-1$, then act optimally)}\\
        &= \mathbb{E}_{\mu,\, a_{s-1} \leftarrow a}\Big[\sum_{k=s-1}^{t-1} \gamma^{k-s+1} r_k \,\Big|\, H_{s-1} \Big] \\&\qquad + \mathbb{E}_{\mu,\, a_{s-1} \leftarrow a}\Big[\gamma^{t-s+1}\, U_t([t] \mid [t]) \,\Big|\, H_{s-1}\Big]\\    
        &= \mathbb{E}_{\mu}\Big[\sum_{k=s-1}^{t-1} \gamma^{k-s+1} r_k \,\Big|\, H_{s-1}, \, a_{s-1} = a\Big] \\&\qquad + \mathbb{E}_{\mu}\Big[\gamma^{t-s+1}\, U_t([t] \mid [t]) \,\Big|\, H_{s-1}, a_{s-1} = a \Big] && \text{($\mu_{s-1}$ depends only on $H_{s-1}$, full support)}\\
        &= \mathbb{E}_{\mu}\Big[\sum_{k=s-1}^{t-1} \gamma^{k-s+1} r_k \,\Big|\, H_{s-1}, \, a_{s-1} = a\Big] + \gamma^{t-s+1}\, U_t([t] \mid B_s) 
    \end{aligned}
\end{equation}

The last line followed from Eq.~\eqref{eq:def_U} and the tower property, with $B_s \subseteq [t]$.

% old version with false inequality
% \begin{equation*}\label{eq:q_bound_1}
%     \begin{aligned}
%         Q_{s-1}^*(H_{s-1},a) = \mathbb E_{\pi^*, a_{s-1} = a}[\sum_{k=s-1}^{t-1}\gamma^{k-s+1}r_k|H_{s-1}] + \gamma^{t-s+1} U_t([t] ~|~ [s-1]) \\
%         \geq \mathbb E_{\pi}[\sum_{k=s-1}^{t-1}\gamma^{k-s+1}r_k~|~H_{s-1}] + \gamma^{t-s+1} U_t([t] ~|~ [s-1]) ~\text{ for any policy }~\pi % should I leave that for any policy??
%     \end{aligned}
% \end{equation*}

Now we find lower bound of $U_t([t] ~|~ B_s)$:

% \begin{equation*} %old shitty (before b detended on a)
%     \begin{aligned}\label{eq:u_bound_1}
%     U_t([t] ~|~ [s-1]) &\geq U_t([t] \setminus R ~|~ [s-1]) && \text{by Prop. \ref{???}} \\
%     &=U_t([t] \setminus R ~|~ [s-1]) + U_t([t] \setminus (R \cup \{s\}) ~|~ [s-1]) - U_t([t] \setminus (R \cup \{s\}) ~|~ [s-1])  \\
%     &= \Lambda_t(\{s\};[t] \setminus (R \cup \{s\})) + U_t([t] \setminus (R \cup \{s\}) ~|~ [s-1]) \\&\qquad \text{(as $s \notin R \implies \left[ [t]\setminus R = \{s\} \cup [t]\setminus R\right]$)} \\
%     \end{aligned} 
% \end{equation*}

\begin{equation}\label{eq:u_bound_2} % new
\begin{aligned}
U_t([t] \mid B_s)
&\overset{(a)}\ge U_t([t] \setminus R \mid B_s)
&& \text{(Prop.~\ref{prop:black}, $B_s \subseteq [t] \setminus R$)}\\
&= U_t([t] \setminus R \mid B_s) - U_t(C_R \mid B_s) + U_t(C_R \mid B_s)
&& \text{(add and subtract)}\\
&= \Lambda_t(\{s\};\ C_R \mid B_s) + U_t(C_R \mid B_s)
&& \text{(Eq.~\eqref{eq:Lambda}, $C_R \cup \{s\} = [t] \setminus R$)}
\end{aligned}
\end{equation}

By plugging Eq. \eqref{eq:u_bound_2} into Eq. \eqref{eq:plain_floor}, we obtain:

\begin{equation*}
    \begin{aligned}
        &Q_{s-1}^*(H_{s-1},a)\\ 
        &\geq \mathbb E_{\mu}\big[\sum_{k=s-1}^{t-1}\gamma^{k-s+1}r_k~|~H_{s-1}, a_{s-1}=a\big] + \gamma^{t-s+1} \big[\Lambda_t(\{s\};\ C_R \mid B_s) + U_t(C_R \mid B_s)\big] \\    
        & = L_{s,t,R}(a) + \gamma^{t-s+1}\Lambda_t(\{s\};\ C_R \mid B_s) 
    \end{aligned}
\end{equation*}

\end{proof}

\begin{corollary}\label{corr:L_b_bound}
Fix $t \geq s$ and an action $a$,
\begin{equation} 
    Q_{s-1}^*(H_{s-1},a)\geq L_{s,t}(a) + \gamma^{t-s+1} b_{s,t}(a)
\end{equation}
\begin{proof}

Recall the definition of $L_{s,t}$ in Eq.~\eqref{eq:def_Lst}:

\begin{equation}
\begin{aligned}
L_{s,t}(a) &:= \mathbb{E}_\mu\Big[\sum_{k=s-1}^{t-1} \gamma^{k-s+1} r_k \,\Big|\, H_{s-1},\ a_{s-1} = a\Big] + \gamma^{t-s+1}\, U_t\Big([s-1] \cup \{\alpha_{s-1}\} \,\Big|\, [s-1] \cup \{\alpha_{s-1}\}\Big)\\
&= \mathbb{E}_\mu\Big[\sum_{k=s-1}^{t-1} \gamma^{k-s+1} r_k \,\Big|\, H_{s-1},\ a_{s-1} = a\Big] + \gamma^{t-s+1}\, U_t\Big(B_s \,\Big|\, B_s\Big) \\&\qquad \text{(with $B_s$, $C_R$ as in Prop. \ref{prop:acq})}.
\end{aligned}
\end{equation}

Note that for all $R \subseteq \{s+1, \dots, t\}$:

\begin{equation}
    C_R \supseteq C_{\{s+1, \dots, t\}} = B_s
\end{equation}

Proposition \ref{prop:black} with $A = B_s$, $S = C_R$ and $B = B_s$ gives $U_t(C_R \mid B_s) \ge U_t(B_s \mid B_s)$. The reward terms of $L_{s,t,R}(a)$ and $L_{s,t}(a)$ are identical, so
\begin{equation}\label{eq:L_mono}
    L_{s,t,R}(a) \geq L_{s,t}(a)
\end{equation}
therefore, 
\begin{equation}
\begin{aligned}
Q_{s-1}^*(H_{s-1},a)\;&\ge\; \sup_{R \subseteq\{s+1, \dots, t\}} L_{s,t,R}(a)+\gamma^{t-s+1}\,\Lambda_t\big(\{s\};C_R|B_s\big) && \text{(Prop.~\ref{prop:acq})}\\ 
&\geq \sup_{R \subseteq\{s+1, \dots, t\}} L_{s,t}(a)+\gamma^{t-s+1}\,\Lambda_t\big(\{s\};C_R|B_s\big)\\
&\geq L_{s,t}(a) + \gamma^{t-s+1} b_{s,t}(a) && \text{(Eq.~\eqref{eq:b})}
\end{aligned}
\end{equation}
\end{proof}
\end{corollary}

% \begin{remark}\label{rem:quick_note}
% For fixed $t$, the sum $L_{s,t,R}(a) + \gamma^{t-s+1}\Lambda_t(\{s\};\ C_R \mid B_s)$ is largest at $R = \varnothing$, see Eq.~\eqref{eq:u_bound_2} (a). $R$ only decides how this value splits into a baseline and the credit of $h_s$. The floor of Corollary~\ref{corr:L_b_bound} is therefore never tighter than the plain floor, the last line of  Eq.~\eqref{eq:plain_floor}. Even though it might therefore appear counterintuitive, we conceptually expect an advantage to using it. As Figure \ref{fig:voi} aims to make explicit, we expect the estimate of $\Lambda_t$ to become accurate sooner than that of $U_t([t] \mid B_s)$ when $\Lambda_t \gg 0$. The baseline $U_t(B_s \mid B_s)$ in $L_{s,t}$ masks everything acquired from $s$ onward and can be fitted from every episode reaching $t$, so the part of the floor that depends on rare encounters of $h_s$ is confined to $b_{s,t}$.
% \end{remark}

\begin{remark}[Why split the floor]\label{rem:quick_note}
The last line of Eq.~\eqref{eq:plain_floor} gives a second floor,
\[
P_{s,t}(a) := \mathbb{E}_\mu\Big[\textstyle\sum_{k=s-1}^{t-1}\gamma^{k-s+1} r_k \,\Big|\, H_{s-1},\, a_{s-1}=a\Big] + \gamma^{t-s+1}\, U_t\big([t]\mid B_s\big),
\]
which is the value of following $\mu$ until $t-1$ and then acts optimally. It is the history-based analogue of the expected $n$-step lower bound of \citet{he2016learning} with $n=t-s+1$. Write $S_{s,t}(a) := L_{s,t}(a) + \gamma^{t-s+1} b_{s,t}(a)$ for the split floor of Corollary~\ref{corr:L_b_bound}. Both floors start from the same baseline $L_{s,t}(a)$ and differ only in what they add on top of it:
\begin{align*}
P_{s,t}(a) &= L_{s,t}(a) + \gamma^{t-s+1}\,\Lambda_t\big(\{s,\dots,t\};\, B_s \mid B_s\big),\\
S_{s,t}(a) &= L_{s,t}(a) + \gamma^{t-s+1}\sup_{R\subseteq\{s+1,\dots,t\}} \Lambda_t\big(\{s\};\, C_{s,t,R} \mid B_s\big).
\end{align*}
$P$ adds the value of everything acquired from $s$ onward, while $S$ adds only the value of $h_s$. By Eqs.~\eqref{eq:plain_floor}, \eqref{eq:u_bound_2}(a) and \eqref{eq:L_mono}, $P$ is at least as tight for exact quantities.

We nevertheless expect $S$ to be easier to estimate. The value that $P$ adds depends on everything $\mu$ happens to observe after $s$, including rare observations that have nothing to do with the action $a_{s-1}$. Across episodes, this could make its estimate noisy. The value that $S$ adds is zero in every episode in which $h_s$ does not change the action chosen at $t$. This matters because Eq.~\eqref{eq:update} takes a maximum for a noisy floor estimate $\hat F$,
\[
\mathbb{E}\big[\max\{TQ, \hat F\}\big] \;\ge\; \max\{TQ, \mathbb{E}[\hat F]\},
\]
so noise in the floor turns into upward bias of the target and can push it above $Q^*$. In Fig.~\ref{fig:voi}, for example, $P_{k,t}$ lifts the skip at $k-1$ by the expected value of $x_2$, which $\mu$ acquires only later and possibly rarely at $u-1$, whereas $S$ lifts only the looks at $s-1$ and $u-1$.

$S$ could thus trade tightness for lower variance and for attribution to the observation that earned the value. Importantly, this argument still has to be made formal by deriving the corresponding variance terms.
\end{remark}
% \textit{Redundancy.}
% If $h_s$ is acquired again before $t$, neither acquisition alone changes the
% decision. Each acquiring action is evaluated given what was known when it was
% taken: the later one has $h_s$ in its base set and earns nothing, so only the
% earlier one can carry the credit. Placing the repeat in $R$ grants this credit
% in full, rather than discounted by the chance that $\mu$ repeats the
% acquisition. This keeps the floor valid and does not force early acquisition:
% where $h_s$ was skipped, the later acquisition is the first and is credited.
% Which acquisition is preferable, e.g.\ by cost, is decided by $Q^*$ through the
% rewards. Until convergence, however, the split floor favours early acquisition,
% since a skip benefits from a later acquisition's credit only through
% bootstrapping.

%% file: supplementary.tex
% Requires \usepackage{booktabs} and \usepackage{hyperref} (or \usepackage{url} for \url).
\section{Supplementary Material}\label{app:supplementary}

\paragraph{Code.}
The implementation of Memory Navigation and of all experiments is available at
\url{https://github.com/daniel-strauss/memory_navigation.git}.

\paragraph{Rollout videos.}
Table~\ref{tab:videos} links rollout videos of the agents compared in the Memory Navigation paragraph.
The last two groups show how the agent without a retrieval head circumvents its short context:
by waiting a cue-dependent number of steps, it stores the cue in its temporal embedding;
with all positional embeddings removed and a context of 3 tokens, it instead follows a cue-dependent
path and so stores the cue in its recent observations.

\begin{table}[h]
\centering
\caption{Rollout videos on MiniGrid-Memory-S17Random. Purple denotes attention (scaled in the first panel),
the red triangle the agent, and the orange rectangle the Memory Ghost. Panels from left to right:
environment state with the agent's $3\times3$ observation grid; Transformer activations with temporal
and spatial positional encodings; memory states (single: observed, double: visited); token consumption timeline.}
\label{tab:videos}
\begin{tabular}{l}
\toprule
\textit{Baseline} \\
\quad \href{https://youtu.be/T6GYvecp2uo}{Transformer-XL without retrieval head} \\
\midrule
\textit{Retrieval heads} \\
\quad \href{https://youtu.be/1ooM1cefwcA}{Memory Navigation} \\
\quad \href{https://youtu.be/WNiA99s-hug}{Absolute position, categorical head} \\
\quad \href{https://youtu.be/KNbboIBTi5Q}{Absolute position, Gaussian head} \\
\midrule
\textit{Cue stored in the temporal embedding (no retrieval head, context 5)} \\
\quad \href{https://youtu.be/hYS3fTfLJns}{Cue: circle} \\
\quad \href{https://youtu.be/xSlaKNFGN0E}{Cue: key} \\
\midrule
\textit{Cue stored in recent observations (no positional embeddings, context 3)} \\
\quad \href{https://youtu.be/3_Pj8aRWMsk}{Cue: circle} \\
\quad \href{https://youtu.be/1uhSnoPl6cI}{Cue: key} \\
\bottomrule
\end{tabular}
\end{table}

\newpage